\documentclass[lettersize,journal]{IEEEtran}
\usepackage{amsmath,amsfonts}
\usepackage{algorithmic}
\usepackage{algorithm}
\usepackage{array}
\usepackage[caption=false,font=normalsize,labelfont=sf,textfont=sf]{subfig}
\usepackage[colorlinks,urlcolor=blue,linkcolor=blue,citecolor=blue]{hyperref}
\usepackage{textcomp}
\usepackage{stfloats}
\usepackage{url}
\usepackage{verbatim}
\usepackage{graphicx}
\usepackage[numbers]{natbib}
\usepackage{amsthm,amssymb}
\usepackage{booktabs}
\usepackage{multirow}

\newtheorem{theorem}{Theorem}
\newtheorem{lemma}{Lemma}
\newtheorem{corollary}{Corollary}

\newtheorem{remark}{Remark}
\newtheorem{definition}{Definition}

\DeclareMathOperator{\VC}{VC}
\DeclareMathOperator{\Pdim}{Pdim}
\DeclareMathOperator{\KL}{KL}
\DeclareMathOperator{\TV}{TV}

\newcommand{\mathbbm}[1]{\mathbf{#1}}

\begin{document}

\title{How Many Independent Samples Does a Satellite Image Contain? Generalization Bounds for Spatially Dependent Data}

\author{Robin Young\\
Department of Computer Science and Technology\\
University of Cambridge, UK \\
ray25@cam.ac.uk}

\markboth{}{}
\pagestyle{plain}
\thispagestyle{plain}

\maketitle

\begin{abstract}
Machine learning classifiers for remote sensing imagery are typically evaluated as though every pixel were an independent sample. Spatial autocorrelation violates this assumption, since neighboring pixels carry redundant information which inflates sample sizes. How many independent samples does a satellite image actually contain? For an $n \times n$ image whose spatial correlation persists over a range of $r$ pixels, the effective sample size is $\Theta(n^2/r^2)$, not $n^2$. We prove this as a finite-sample upper bound for classifiers on spatially correlated data, and show via a matching lower bound that the rate is tight, and no algorithm can do better. We extend the results to images with directional correlation and spatially varying correlation structure. Our result justifies spatial cross-validation since block holdout with separation proportional to the correlation range achieves optimal generalization guarantees, while random holdout can underestimate confidence interval widths by a factor proportional to $r$. We validate the theory on synthetic data and satellite image tiles from three sensors (Landsat 8, Sentinel-2, and Sentinel-1).
\end{abstract}


\section{Introduction}
\label{sec:intro}

A persistent methodological issue in remote sensing is spatial autocorrelation, namely the mismatch between the number of labeled pixels available and the amount of statistically independent information they contain. An $n \times n$ satellite image nominally provides~$n^2$ labeled samples, but neighboring pixels are spatially correlated. Adjacent forest pixels look alike, adjacent urban pixels look alike, because a scene is composed of internally homogeneous objects (fields, forest stands) whose size and spacing set the distance over which correlation persists~\cite{jupp1988autocorrelation}. Standard learning-theoretic guarantees, that is, high-probability bounds on how far a classifier's test error can depart from its training error, assume independent and identically distributed samples and therefore do not apply.

In practice, this manifests as a well-documented bias. Models evaluated by random holdout on spatially autocorrelated imagery report higher accuracy than models evaluated by spatially structured holdout~\cite{roberts2017cross, wadoux2021spatial, ploton2020spatial, meyer2021predicting}. The recognition that covariance structure differs across strata, and that treating a scene as globally homogeneous is inadequate, has a long history in geostatistics~\cite{wang2009modeling}. Studies of sampling design under spatial dependence have shown empirically that prediction accuracy degrades sharply once the sampling interval approaches the autocorrelation range~\cite{rocha2021role}. Concurrently, Liao et al.~\cite{liao2026scaling} have established an empirical scaling law relating sample size to classifier performance in remote sensing. Our results complement theirs by providing a theoretical explanation for why the effective sample size entering such scaling relationships is $\Theta(n^2/r^2)$ rather than the nominal pixel count. The remote sensing community has converged on the recommendation to use spatial cross-validation, but a characterization of the cost of ignoring spatial structure, how much independent information an image truly contains, and a proof that the standard spatial CV prescription is optimal has remained open. Meanwhile, the statistical learning theory literature has developed generalization bounds for dependent data under mixing conditions~\cite{mohri2010stability, kuznetsov2017generalization, modha1996minimum}, but almost exclusively for temporally indexed (one-dimensional) sequences. \footnote{Readers unfamiliar with statistical learning theory may find the primer in the appendix helpful. It glosses the VC-dimension and uniform-convergence machinery used below.}

We bridge this gap. Our contributions are a finite-sample generalization bound for classifiers on a two-dimensional $\beta$-mixing random field, with an explicit dependence on the spatial correlation range~$r$, a matching minimax lower bound showing the rate is tight up to logarithmic factors, and a corollary providing theoretical justification for spatial cross-validation in remote sensing and quantifying the overoptimism of random holdout. For a typical Sentinel-2 scene with correlation range 30 pixels, the effective sample size may be much smaller than the pixel count, meaning reported confidence intervals could be too narrow.

\section{Setup}
\label{sec:setup}

Let $\{(X_s, Y_s)\}_{s \in \mathbb{Z}^2}$ be a strictly stationary random field, where $X_s \in \mathcal{X} \subseteq \mathbb{R}^d$ is a feature vector (e.g. reflectance) and $Y_s \in \{1, \dots, K\}$ is a land cover label. We observe this field on the lattice $\Lambda_n = \{1, \dots, n\}^2$.

\begin{definition}[Spatial $\beta$-mixing with latent alignment]
\label{def:mixing}
Let $\tau$ be a random variable on a finite set $\mathcal{T}$ (possibly a single point, in which case all conditioning below is vacuous\footnote{This will be needed to ensure stationarity in the lower-bound construction. For the upper bound the reader may take $\tau$ to be degenerate and ignore the conditioning.}) such that, conditional on $\tau = t$, the single-site law of $(X_s, Y_s)$ is the same for all $s$ and $t$. For finite $S, T \subset \mathbb{Z}^2$ with $d_\infty(S,T) = \min_{s \in S, t' \in T} \|s-t'\|_\infty \geq g$, let $\beta(\sigma_S, \sigma_T \mid \tau{=}t) = \mathbb{E}\big[\sup_{B \in \sigma_T} |P(B \mid \sigma_S, \tau{=}t) - P(B \mid \tau{=}t)| \,\big|\, \tau{=}t\big]$, where $\sigma_S = \sigma(\{(X_s,Y_s): s \in S\})$, and set
\begin{equation}
\beta(g) = \sup_{t \in \mathcal{T}}\; \sup_{\substack{S, T \text{ finite} \\ d_\infty(S,T) \geq g}} \beta(\sigma_S, \sigma_T \mid \tau = t).
\end{equation}
The mixing coefficient measures how close to independent two regions of the image become as you move them apart. For degenerate $\tau$ this is the standard spatial $\beta$-mixing coefficient. The alignment $\tau$ models a global nuisance given which the field mixes (for instance, the unknown phase of the object grid like field boundaries or forest stands relative to the pixel lattice). No analyst knows where scene objects fall relative to the sensor grid, so this is a natural feature of remote sensing data rather than a technical artifact. The single-site condition ensures the risk defined below is independent of $s$ and $\tau$.
\end{definition}

Let $\mathcal{H}$ be a hypothesis class of classifiers $h\colon \mathcal{X} \to \{1,\dots,K\}$ with VC dimension $d_{\mathrm{vc}} < \infty$. Define the 0--1 loss $\ell(h,s) = \mathbbm{1}[h(X_s) \neq Y_s]$ and the true risk $R(h) = \mathbb{E}[\ell(h,s)]$.

For a stationary Gaussian random field with covariance $C(\mathbf{h}) = \sigma^2 e^{-\|\mathbf{h}\|/r}$ (exponential model with range~$r$) and spectral density bounded away from zero, we have $\beta(g) \leq c_0 e^{-c_1 g/r}$ for constants $c_0, c_1 > 0$~\cite{doukhan1994mixing, bradley2007mixing}, with degenerate $\tau$. For the Mat\'ern model with smoothness~$\nu$, polynomial mixing rates $\beta(g) \sim g^{-\alpha}$ hold with $\alpha$ depending on $\nu$ and the spatial dimension.

\section{Upper Bound}
\label{sec:upper}

The challenge is that standard generalization bounds assume independent samples, yet nearby pixels in a satellite image are correlated. Our proof strategy converts the dependent-data problem into an independent-data problem at a controlled cost. We partition the $n \times n$ image into small blocks separated by gaps of width~$g$ (Fig.~\ref{fig:blocking}). If $g$ is large enough relative to the correlation range~$r$, the blocks are nearly independent, and we can apply standard concentration inequalities as if the blocks were i.i.d. samples. The cost is that the gaps consume pixels and an image with $n^2$ pixels yields only $m \approx n^2/g^2$ usable blocks. The proof optimizes this tradeoff where wider gaps give better independence but fewer blocks and shows that the optimal gap is $g = \Theta(r \log n)$, yielding $m = \Theta(n^2/r^2)$ effectively independent blocks up to logarithmic factors.

\begin{figure}[htbp]
\centering
\includegraphics[width=0.8\columnwidth]{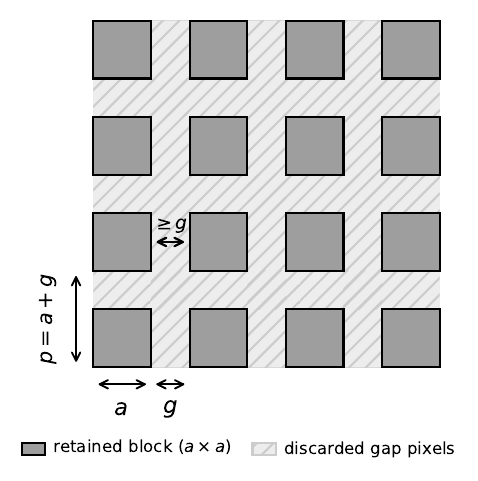}
\caption{Graphical illustration of the blocking construction in Thm.~\ref{thm:upper}.}
\label{fig:blocking}
\end{figure}

Formally, the argument has three components: a coupling lemma that bounds the total variation distance between the dependent blocks and an i.i.d. surrogate (Lemma~\ref{lem:coupling}), a complexity bound on the block-averaged loss class (Lemma~\ref{lem:pdim}), and an optimization over the block geometry.

Fix integers $a \geq 1$ (block side length) and $g \geq 1$ (gap width). Set $p = a + g$. Define non-overlapping blocks
\begin{equation}
B_{ij} = \bigl\{(ip+1,\dots,ip+a) \times (jp+1,\dots,jp+a)\bigr\}
\end{equation}
for $i,j = 0, \dots, M{-}1$, where $M = \lfloor n/p \rfloor$. This yields $m = M^2$ blocks, each containing $a^2$ pixels, with any two distinct blocks separated by $d_\infty(B_{ij}, B_{i'j'}) \geq g$. Define the block empirical risk
\begin{equation}
\hat{R}_{\mathrm{block}}(h) = \frac{1}{m} \sum_{i,j=0}^{M-1} Z_{ij}(h), \quad Z_{ij}(h) = \frac{1}{a^2}\sum_{s \in B_{ij}} \ell(h,s).
\end{equation}
By stationarity, $\mathbb{E}[\hat{R}_{\mathrm{block}}(h)] = R(h)$. The standard telescoping argument~\cite{yu1994rates} extends directly to the 2D setting:

\begin{lemma}[Coupling]
\label{lem:coupling}
Conditionally on $\tau = t$, for every $t \in \mathcal{T}$, there exist i.i.d.\ random variables $\{Z'_{ij}\}$ with $Z'_{ij} \overset{d}{=} Z_{ij}$ such that
\begin{equation}
d_{\TV}\!\bigl(\mathcal{L}(Z_{ij}: i,j),\; \textstyle\prod_{i,j} \mathcal{L}(Z_{ij})\bigr) \leq (m-1)\beta(g).
\end{equation}
\end{lemma}

\begin{proof}
Enumerate the blocks as $Z_1, \dots, Z_m$; all laws below are conditional on $\tau = t$. By the triangle inequality, $d_{\TV}(\mu_0, \mu_m) \leq \sum_{k=2}^m d_{\TV}(\mathcal{L}(Z_k \mid Z_1,\dots,Z_{k-1}), \mathcal{L}(Z_k))$. Each $Z_k$ is $\sigma_{B_k}$-measurable, $\bigcup_{j<k} B_j$ is finite at $\ell_\infty$-distance $\geq g$ from $B_k$, so each term is $\leq \beta(g)$.
\end{proof}

\begin{lemma}[Complexity of Block-Averaged Class]
\label{lem:pdim}
The pseudo-dimension of $\mathcal{F} = \{(x_s,y_s)_{s \in B} \mapsto \frac{1}{a^2}\sum_{s \in B}\ell(h,s) : h \in \mathcal{H}\}$ satisfies $\Pdim(\mathcal{F}) \leq 2 d_{\mathrm{vc}} \log_2(ea^2)$.
\end{lemma}

\begin{proof}
Shattering $k$ block-samples with thresholds requires all $2^k$ sign patterns of $\frac{1}{a^2}\sum_s \ell(h,s) - t_i$; each $h$ induces a binary labeling of $ka^2$ pixels, so $2^k \leq S_{\mathcal{H}}(ka^2) \leq (eka^2/d_{\mathrm{vc}})^{d_{\mathrm{vc}}}$ by Sauer's lemma. Then $k \leq d_{\mathrm{vc}}\log_2(eka^2/d_{\mathrm{vc}})$, and $x \leq A\log_2(Bx) \Rightarrow x \leq 2A\log_2(AB)$ for $A,B \geq 1$ gives the claim.
\end{proof}

\begin{remark}[On generality]
We emphasize that the correlation range $r$ and the number of pixels $n^2$ enter the effective sample size $n_{\mathrm{eff}} = \Theta(n^2/r^2)$ independently of the classifier. The hypothesis class affects the bound only through its VC dimension $d_{\mathrm{vc}}$, which multiplies the deviation as $\sqrt{d_{\mathrm{vc}}}$ but does not change the scaling in $r$. The rate is therefore a property of the spatial dependence structure, not of any particular learning algorithm.
\end{remark}

\begin{theorem}[Upper bound]
\label{thm:upper}
Let $\{(X_s,Y_s)\}$ be a strictly stationary random field on $\Lambda_n = \{1,\dots,n\}^2$ satisfying Definition~\ref{def:mixing} with $\beta(g) \leq c_0 e^{-c_1 g/r}$ for constants $c_0,c_1 > 0$ and correlation range $r \geq 1$. Let $\mathcal{H}$ have VC dimension~$d_{\mathrm{vc}}$. For any $\delta \in (0,1)$, choosing block side length $a$ and gap $g$ optimally, with probability at least $1-\delta$:
\begin{equation}
\sup_{h \in \mathcal{H}} \bigl|R(h) - \hat{R}_{\mathrm{block}}(h)\bigr|
\leq C_1 \frac{r \log(n/\delta)}{n}\sqrt{d_{\mathrm{vc}}\log n},
\end{equation}
where $C_1$ depends only on $c_0$ and~$c_1$. Equivalently, the effective sample size is $n_{\mathrm{eff}} = \Theta\!\big(n^2 / (r^2 \log^2 n)\big)$.
\end{theorem}

\begin{proof}
Fix $t \in \mathcal{T}$ and argue conditionally on $\tau = t$. Combining Lemmas~\ref{lem:coupling} and~\ref{lem:pdim} with the standard pseudo-dimension bound for i.i.d.\ bounded variables:
\begin{align}
&P\!\left(\sup_{h \in \mathcal{H}} |R(h) - \hat{R}_{\mathrm{block}}(h)| > \epsilon \,\Big|\, \tau = t\right) \nonumber\\
&\quad\leq 8\!\left(\frac{4em}{d'}\right)^{\!d'}\!\exp\!\left(-\frac{m\epsilon^2}{32}\right) + (m-1)\beta(g)
\label{eq:combined}
\end{align}
where $m = \lfloor n/(a+g)\rfloor^2$ and $d' \leq 2d_{\mathrm{vc}}\log_2(ea^2)$.

\textit{Optimization.} The coupling constraint $(m-1)\beta(g) \leq \delta/2$ requires $g \geq \frac{r}{c_1}\log\frac{2c_0 m}{\delta} = \Theta(r\log(n/\delta))$; the concentration term scales as $\sqrt{d'\log m/m}$. We set $a = 1$: single-pixel blocks may seem counterintuitive, since averaging over larger blocks reduces each block's variance, but larger blocks also inflate the pseudo-dimension (Lemma~\ref{lem:pdim}) and reduce the block count for a given gap, and the variance reduction cannot compensate for these two costs. With $a=1$, $d' = d_{\mathrm{vc}}$ and $m = \Theta(n^2/(r^2\log^2(n/\delta)))$ is maximized.

With $a = 1$, each $Z_{ij}(h) = \ell(h, s_{ij})$ is a single-site loss, whose conditional law is independent of $t$ with mean $R(h)$ by the single-site condition of Definition~\ref{def:mixing}. Hence the coupled surrogates are conditionally i.i.d.\ with mean $R(h)$, the right side of~\eqref{eq:combined} is uniform in $t$, and averaging over $\tau$ preserves the bound. Setting~\eqref{eq:combined} equal to $\delta$ yields the result.
\end{proof}

\section{Lower Bound}
\label{sec:lower}

To show that no algorithm can beat the $\Theta(n^2/r^2)$ effective sample size, we construct a worst-case random field that is stationary, satisfies the mixing condition, and yet contains $\lfloor n/r \rfloor^2$ independent pieces of information. The construction tiles the image with $r \times r$ ``super-pixels,'' each carrying an independent draw from a hard classification instance. All pixels within a super-pixel share the same feature and label, so observing the full $n^2$ pixels is no more informative than observing the $\lfloor n/r \rfloor^2$ super-pixel summaries. A random spatial offset~$\tau$ ensures strict stationarity, and conditionally on $\tau$ the field is block-independent, so the construction satisfies Definition~\ref{def:mixing} with $\tau$ as the latent alignment, the offset is the unknown grid phase that the definition models. Applying Assouad's lemma to the resulting i.i.d.\ super-pixel problem gives a minimax lower bound of $\Omega(r\sqrt{d_{\mathrm{vc}}}/n)$, matching the upper bound up to logarithmic factors.

\begin{theorem}[Minimax lower bound]
\label{thm:lower}
Let $\mathcal{H}$ have VC dimension $d_{\mathrm{vc}} \geq 1$. For any $r \geq 1$ and $n \geq 2r\sqrt{d_{\mathrm{vc}}}$, there exists a strictly stationary random field on $\Lambda_n$ satisfying Definition~\ref{def:mixing} with $\beta(g) = 0$ for all $g \geq r$ hence lying in the model class of Theorem~\ref{thm:upper} such that
\begin{equation}
\inf_{\hat{h}}\; \sup_{P} \;\mathbb{E}\bigl[R(\hat{h}) - R(h^*_P)\bigr] \geq \frac{1}{12}\frac{r\sqrt{d_{\mathrm{vc}}}}{n}
\end{equation}
where the infimum is over all learning algorithms and the supremum is over all distributions within the model class.
\end{theorem}

\begin{proof}
Since $\VC(\mathcal{H}) = d_{\mathrm{vc}}$, there exist points $z_1, \dots, z_{d_{\mathrm{vc}}} \in \mathcal{X}$ shattered by $\mathcal{H}$.

\textit{Step 1: Construction.} Let $\tau$ be uniform on $\{0,\dots,r{-}1\}^2$, independent of all other randomness. Define super-pixel assignments $\kappa(s) = \lfloor(s - \tau)/r\rfloor$. For each super-pixel $k$, draw independently $J_k \sim \mathrm{Uniform}\{1,\dots,d_{\mathrm{vc}}\}$ and $\xi_k \sim \mathrm{Bernoulli}(\eta)$, and set $X_s = z_{J_{\kappa(s)}}$, $Y_s = c_{J_{\kappa(s)}} \oplus \xi_{\kappa(s)}$, where $c \in \{0,1\}^{d_{\mathrm{vc}}}$ is the unknown concept.

\textit{Step 2: Stationarity.} Shifting $s$ by $v \in \mathbb{Z}^2$ is equivalent to shifting $\tau$ by $-v$; since $\tau$ is uniform on a period-$r$ grid, the field is strictly stationary.

\textit{Step 3: Mixing.} Condition on $\tau = t$, so $\kappa_t(s) = \lfloor(s-t)/r\rfloor$ is deterministic. The conditional single-site law is that of $(z_J, c_J \oplus \xi)$ with $J \sim \mathrm{Uniform}\{1,\dots,d_{\mathrm{vc}}\}$, $\xi \sim \mathrm{Bernoulli}(\eta)$, independent of $s$ and $t$. If $d_\infty(S,T) \geq r$, the index sets $\kappa_t(S)$ and $\kappa_t(T)$ are disjoint, and since the marks $(J_k,\xi_k)$ are i.i.d.\ across super-pixels, $\beta(g \mid \tau{=}t) = 0$ for $g \geq r$, uniformly in $t$. (Unconditionally the field is \emph{not} mixing since observing one region reveals the grid phase, which is informative about joint laws in distant regions. Conditioning on $\tau$ removes this nuisance.)

\textit{Step 4: Sufficient statistic.} Reveal $\tau$ to the learner; this only strengthens it, so any lower bound survives. Conditionally on $\tau$, the data reduce to $N \geq \lfloor n/r\rfloor^2$ i.i.d.\ pairs $(J_k, Y_k)$, one per super-pixel, and the reduction does not depend on $c$.

\textit{Step 5: Assouad's lemma.} Work conditionally on $\tau$; the excess risk is $\tau$-free by Step~3, so the bound below is uniform in $\tau$ and averaging gives the theorem. The $2^{d_{\mathrm{vc}}}$ concepts form a hypercube. Flipping bit~$j$ changes the excess risk by $\alpha_j = (1-2\eta)/(2d_{\mathrm{vc}})$. The number of super-pixels informative about bit~$j$ is $N_j \sim \mathrm{Binomial}(N, 1/d_{\mathrm{vc}})$, and by Pinsker, $\TV_j \leq \sqrt{N_j\KL(\eta\|1{-}\eta)/2}$. Set $\eta = 1/2 - \gamma$ with $\gamma = \frac{1}{6}\sqrt{d_{\mathrm{vc}}/N} \leq 1/12$ (as $N \geq 4d_{\mathrm{vc}}$), so $\KL(\eta\|1{-}\eta) \leq \frac{(1-2\eta)^2}{\eta(1-\eta)} \leq 17\gamma^2$. Since $\TV_j^2 \leq 17\gamma^2 N_j/2$ given $N_j$ and $\sqrt{\cdot}$ is concave, Jensen gives $\mathbb{E}[\TV_j] \leq \sqrt{17\gamma^2\mathbb{E}[N_j]/2} = \sqrt{17\gamma^2 N/(2d_{\mathrm{vc}})} = \sqrt{17/72} < 1/2$. The expectation form of Assouad's lemma~\cite{tsybakov2009introduction} then gives
\begin{equation}
\inf_{\hat{h}} \max_c \mathbb{E}[R(\hat{h}) - R(h^*_c)]
\geq d_{\mathrm{vc}}\alpha_j \cdot \tfrac{1}{2}
= \tfrac{\gamma}{2}
= \tfrac{1}{12}\sqrt{\tfrac{d_{\mathrm{vc}}}{N}},
\end{equation}
and $N \leq n^2/r^2$ completes the proof.
\end{proof}

Comparing Theorems~\ref{thm:upper} and~\ref{thm:lower}, the upper bound is $\tilde{O}(r\sqrt{d_{\mathrm{vc}}}/n)$ and the lower bound is $\Omega(r\sqrt{d_{\mathrm{vc}}}/n)$: the rates match and the gap is polylogarithmic.

\section{Extensions to Non-Ideal Settings}
\label{sec:extensions}

The results of Sections~\ref{sec:upper}--\ref{sec:lower} assume an isotropic, globally stationary random field. We now relax both. Section~\ref{sec:aniso} extends the bounds to anisotropic fields with directional correlation, and Section~\ref{sec:piecewise} extends them to piecewise stationary fields with spatially varying correlation structure.

\subsection{Anisotropic Fields}
\label{sec:aniso}

Many remote sensing modalities exhibit directional correlation structure. SAR images have different correlation lengths in range and azimuth, and landscapes shaped by rivers, coastlines, or mountain ridges are naturally anisotropic. We replace the isotropic exponential covariance with

\begin{equation}
C(\mathbf{h}) = \sigma^2 \exp\!\Bigl(-\sqrt{h_1^2/r_1^2 + h_2^2/r_2^2}\Bigr),
\end{equation}
where $r_1, r_2 \geq 1$ are the directional correlation ranges. The corresponding mixing coefficient satisfies
\begin{equation}
\label{eq:aniso_mixing}
\beta(g_1, g_2) \leq c_0 \exp\!\Bigl(-c_1\sqrt{g_1^2/r_1^2 + g_2^2/r_2^2}\Bigr),
\end{equation}
where $g_k$ denotes the separation between two regions in direction~$k$.

The blocking construction generalizes by using rectangular blocks and gaps. Fix unit block size $a_1 = a_2 = 1$ and directional gaps $g_k \geq 1$. Define blocks
\begin{equation}
B_{ij} = \{i(1+g_1)+1\} \times \{j(1+g_2)+1\}
\end{equation}
for $i = 0, \dots, M_1{-}1$ and $j = 0, \dots, M_2{-}1$, where $M_k = \lfloor n/(1+g_k) \rfloor$. This yields $m = M_1 M_2$ single-pixel blocks, any two of which are separated by at least $(g_1, g_2)$ in the respective directions.

\begin{theorem}[Anisotropic upper bound]
\label{thm:aniso_upper}
Let $\{(X_s, Y_s)\}$ be a stationary random field on $\Lambda_n$ satisfying Definition~\ref{def:mixing} with directional rate~\eqref{eq:aniso_mixing} and directional correlation ranges $r_1, r_2 \geq 1$. Let $\mathcal{H}$ have VC dimension~$d_{\mathrm{vc}}$. For any $\delta \in (0,1)$, with probability at least $1 - \delta$:
\begin{equation}
\sup_{h \in \mathcal{H}} \bigl|R(h) - \hat{R}_{\mathrm{block}}(h)\bigr|
\leq C_1 \frac{\sqrt{r_1 r_2}\;\log(n/\delta)}{n}\sqrt{d_{\mathrm{vc}}\log n},
\end{equation}
where $C_1$ depends only on $c_0, c_1$. The effective sample size is $n_{\mathrm{eff}} = \Theta\!\bigl(n^2/(r_1 r_2)\bigr)$ up to logarithmic factors.
\end{theorem}

\begin{proof}
The coupling lemma (Lemma~\ref{lem:coupling}) applies with the same telescoping argument where each pair of blocks is separated by $(g_1, g_2)$, so each coupling step incurs error at most $\beta(g_1, g_2)$. The constraint $(m-1)\beta(g_1,g_2) \leq \delta/2$ is satisfied by setting $g_k = \Theta(r_k \log(n/\delta))$ for each $k$. Since $a_1 = a_2 = 1$, the pseudo-dimension bound (Lemma~\ref{lem:pdim}) gives $d' = d_{\mathrm{vc}}$. The number of blocks is
\begin{equation}
m = \Bigl\lfloor \frac{n}{1+g_1}\Bigr\rfloor \Bigl\lfloor \frac{n}{1+g_2}\Bigr\rfloor = \Theta\!\left(\frac{n^2}{r_1r_2\log^2(n/\delta)}\right).
\end{equation}
Substituting into~\eqref{eq:combined} yields the stated bound. For nondegenerate $\tau$, condition as in Theorem~\ref{thm:upper}. Since $a_1 = a_2 = 1$, the argument is identical.
\end{proof}

\begin{theorem}[Anisotropic lower bound]
\label{thm:aniso_lower}
Under the conditions of Theorem~\ref{thm:lower} with directional ranges $r_1, r_2 \geq 1$ and $n \geq 2\max(r_1,r_2)\sqrt{d_{\mathrm{vc}}}$:
\begin{equation}
\inf_{\hat{h}}\;\sup_P\; \mathbb{E}\bigl[R(\hat{h}) - R(h^*_P)\bigr] \geq \frac{1}{12}\frac{\sqrt{r_1 r_2 d_{\mathrm{vc}}}}{n}.
\end{equation}
\end{theorem}

\begin{proof}
Modify the construction of Theorem~\ref{thm:lower} by using rectangular super-pixels. Let $\tau$ be uniform on $\{0,\dots,r_1{-}1\} \times \{0,\dots,r_2{-}1\}$ and define super-pixel assignments $\kappa(s) = (\lfloor(s_1-\tau_1)/r_1\rfloor, \lfloor(s_2-\tau_2)/r_2\rfloor)$. Each super-pixel has size $r_1 \times r_2$. Stationarity follows from the uniformity of~$\tau$ (shifting by $v$ is equivalent to shifting $\tau$ by $-v$, which preserves the distribution since $\tau_k$ is uniform on a period-$r_k$ grid). Conditionally on $\tau$, super-pixel indices of $S$ and $T$ are disjoint whenever $d_\infty(S,T) \geq \max(r_1,r_2)$, giving $\beta(g \mid \tau) = 0$ for $g \geq \max(r_1,r_2)$ uniformly in $\tau$; the learner is augmented with $\tau$ as in Step~4 of Theorem~\ref{thm:lower}. The number of super-pixels is $N = \lfloor n/r_1\rfloor \cdot \lfloor n/r_2\rfloor \leq n^2/(r_1 r_2)$. The Assouad argument proceeds identically, yielding $\Omega(\sqrt{d_{\mathrm{vc}}/N}) = \Omega(\sqrt{r_1 r_2d_{\mathrm{vc}}}/n)$.
\end{proof}

Theorems~\ref{thm:aniso_upper} and~\ref{thm:aniso_lower} together establish that $n_{\mathrm{eff}} = \Theta(n^2/(r_1 r_2))$ for anisotropic fields, matching up to polylogarithmic factors. The isotropic result is recovered when $r_1 = r_2 = r$.

The practical implication for spatial cross-validation is that block separation should be direction-dependent where folds should be separated by $g_k = \Theta(r_k \log n)$ in direction~$k$, leading to elongated validation blocks aligned with the anisotropy. For SAR imagery, where the range direction may have $r_1 = 50$ pixels and the azimuth $r_2 = 10$, this yields $n_{\mathrm{eff}} = n^2/500$ rather than the isotropic estimate $n^2/2500$ obtained by using $r_{\max} = 50$ in both directions.

\subsection{Piecewise Stationary Fields}
\label{sec:piecewise}

Real remote sensing scenes are rarely globally stationary, and correlation structure varies across land cover types, and the range fitted over cropland may differ substantially from that over forest or urban areas. We extend the theory to fields that are stationary within each of $K$ homogeneous regions.

\begin{definition}[Piecewise stationary field]
\label{def:piecewise}
A random field $\{(X_s, Y_s)\}$ on $\Lambda_n$ is \emph{$(K, \{r_k\})$-piecewise stationary} if there exists a partition $\Lambda_n = \bigcup_{k=1}^K \Lambda_k$ such that, within each region $\Lambda_k$ (containing $|\Lambda_k| = n_k^2$ pixels), the field is stationary and satisfies Definition~\ref{def:mixing} (alignment $\tau_k$) with $\beta_k(g) \leq c_0 e^{-c_1 g/r_k}$, and, conditionally on the tuple $\tau = (\tau_1,\dots,\tau_K)$, the restrictions of the field to distinct regions are mutually independent. 
\end{definition}

Without global stationarity, the expected loss $\mathbb{E}[\ell(h,s)]$ depends on $s$. The natural target becomes the domain-average risk:
\begin{equation}
\label{eq:rbar}
\bar{R}(h) = \frac{1}{n^2}\sum_{s \in \Lambda_n} \mathbb{E}[\ell(h,s)] = \sum_{k=1}^K w_k R_k(h),
\end{equation}
where $w_k = n_k^2/n^2$ and $R_k(h) = \mathbb{E}_{s \sim \Lambda_k}[\ell(h,s)]$ is the within-region risk.

Within each region~$\Lambda_k$, we construct blocks with gap $g_k = \Theta(r_k \log(n/\delta))$, yielding $m_k = \Theta(n_k^2 / (r_k^2 \log^2 n))$ blocks. The per-region block empirical risk is $\hat{R}_k^{\mathrm{block}}(h) = \frac{1}{m_k}\sum_{(i,j) \in \Lambda_k} Z_{ij}^{(k)}(h)$, and the overall estimator is
\begin{equation}
\hat{\bar{R}}(h) = \sum_{k=1}^K w_k\hat{R}_k^{\mathrm{block}}(h).
\end{equation}
By per-region stationarity, $\mathbb{E}[\hat{\bar{R}}(h)] = \bar{R}(h)$.

\begin{theorem}[Piecewise stationary upper bound]
\label{thm:piecewise_upper}
Let $\{(X_s, Y_s)\}$ be a $(K, \{r_k\})$-piecewise stationary field on $\Lambda_n$, with each region $\Lambda_k$ containing $n_k^2$ pixels and having correlation range $r_k \geq 1$. Let $\mathcal{H}$ have VC dimension~$d_{\mathrm{vc}}$. For any $\delta \in (0,1)$, with probability at least $1 - \delta$:
\begin{equation}
\label{eq:piecewise_bound}
\sup_{h \in \mathcal{H}} \bigl|\bar{R}(h) - \hat{\bar{R}}(h)\bigr|
\leq C_2\frac{\log(n/\delta)}{n^2}\sqrt{d_{\mathrm{vc}}\log n \cdot \sum_{k=1}^K n_k^2r_k^2},
\end{equation}
where $C_2$ depends only on $c_0, c_1$. The effective sample size is
\begin{equation}
\label{eq:neff_piecewise}
n_{\mathrm{eff}} = \frac{n^4}{\sum_{k=1}^K n_k^2r_k^2}
\end{equation}
up to logarithmic factors.
\end{theorem}

\begin{proof}
Within each region $\Lambda_k$, apply the coupling lemma (Lemma~\ref{lem:coupling}) with gap $g_k = \Theta(r_k \log(nK/\delta))$ to replace the $m_k$ dependent block variables $\{Z_{ij}^{(k)}\}$ with independent copies. The total coupling error across all regions is at most $\sum_k (m_k - 1)\beta_k(g_k) \leq \delta/2$ by setting the gaps appropriately.

After coupling, $\hat{\bar{R}}(h) - \bar{R}(h)$ is a weighted sum of independent bounded random variables. Specifically, for fixed~$h$:
\begin{equation}
\hat{\bar{R}}(h) - \bar{R}(h) = \sum_{k=1}^K \frac{w_k}{m_k} \sum_{(i,j) \in \Lambda_k} \bigl(Z_{ij}^{(k)}(h) - R_k(h)\bigr).
\end{equation}
Each term $w_k Z_{ij}^{(k)}/m_k$ is bounded in $[0, w_k/m_k]$. By Hoeffding's inequality for independent (not necessarily identically distributed) bounded random variables:
\begin{equation}
P\bigl(|\hat{\bar{R}}(h) - \bar{R}(h)| > \epsilon\bigr) \leq 2\exp\!\left(-\frac{2\epsilon^2}{\sum_{k=1}^K w_k^2/m_k}\right).
\end{equation}

For uniform convergence over $\mathcal{H}$, we use the standard symmetrization and covering number argument, which only requires independence and boundedness. Since $a = 1$ (unit blocks), the loss class has pseudo-dimension $d' = d_{\mathrm{vc}}$ by Lemma~\ref{lem:pdim}. The total number of blocks is $m = \sum_k m_k$, and the covering number of the VC class at scale $\epsilon$ satisfies $\mathcal{N}(\epsilon) \leq (em/d_{\mathrm{vc}})^{d_{\mathrm{vc}}}$. Combining via a union bound over an $\epsilon/2$-cover:
\begin{align}
&P\!\left(\sup_{h \in \mathcal{H}} |\bar{R}(h) - \hat{\bar{R}}(h)| > \epsilon\right) \nonumber\\
&\quad \leq 2\!\left(\frac{2em}{d_{\mathrm{vc}}}\right)^{\!d_{\mathrm{vc}}} \exp\!\left(-\frac{\epsilon^2/2}{\sum_k w_k^2/m_k}\right) + \frac{\delta}{2}.
\label{eq:piecewise_combined}
\end{align}

Substituting $w_k = n_k^2/n^2$ and $m_k = \Theta(n_k^2/(r_k^2\log^2 n))$:
\begin{equation}
\sum_{k=1}^K \frac{w_k^2}{m_k} = \sum_{k=1}^K \frac{n_k^4}{n^4} \cdot \frac{r_k^2\log^2 n}{n_k^2} = \frac{\log^2 n}{n^4}\sum_{k=1}^K n_k^2 r_k^2.
\end{equation}
Setting~\eqref{eq:piecewise_combined} equal to $\delta$ and solving for $\epsilon$ gives the stated bound.
\end{proof}

The effective sample size~\eqref{eq:neff_piecewise} has a natural interpretation. It is the reciprocal of the weighted variance $\sum_k w_k^2/m_k$. Regions with large correlation range $r_k$ contribute fewer effective samples per unit area and thereby reduce the overall effective sample size, even if they occupy a small fraction of the image.

\begin{theorem}[Piecewise stationary lower bound]
\label{thm:piecewise_lower}
Under the conditions of Theorem~\ref{thm:piecewise_upper}, for any region $\Lambda_j$ with $n_j \geq 2r_j\sqrt{d_{\mathrm{vc}}}$:
\begin{equation}
\inf_{\hat{h}}\;\sup_P\;\mathbb{E}\bigl[\bar{R}(\hat{h}) - \bar{R}(h^*_P)\bigr] \geq \frac{w_j}{12}\frac{r_j\sqrt{d_{\mathrm{vc}}}}{n_j} = \frac{n_jr_j\sqrt{d_{\mathrm{vc}}}}{12n^2}.
\end{equation}
In particular, taking the maximum over regions:
\begin{equation}
\label{eq:piecewise_lower}
\inf_{\hat{h}}\;\sup_P\;\mathbb{E}\bigl[\bar{R}(\hat{h}) - \bar{R}(h^*_P)\bigr] \geq \frac{\sqrt{d_{\mathrm{vc}}}}{12n^2}\max_{1 \leq k \leq K} n_kr_k.
\end{equation}
\end{theorem}

\begin{proof}
For a chosen region $\Lambda_j$, construct a distribution in which the classification task within $\Lambda_j$ follows the super-pixel construction of Theorem~\ref{thm:lower} with range $r_j$, and the task is trivial in all other regions (i.e. $R_k(h) = R_k(h^*)$ for $k \neq j$ under every hypothesis). Then $\bar{R}(\hat{h}) - \bar{R}(h^*) = w_j(R_j(\hat{h}) - R_j(h^*))$, and the minimax rate for estimating $R_j$ over $\Lambda_j$ alone is $\Omega(r_j\sqrt{d_{\mathrm{vc}}}/n_j)$ by Theorem~\ref{thm:lower}. Multiplying by $w_j = n_j^2/n^2$ gives the per-region lower bound. Taking the maximum over $j$ yields~\eqref{eq:piecewise_lower}.
\end{proof}

Comparing the upper bound~\eqref{eq:piecewise_bound} and the lower bound~\eqref{eq:piecewise_lower}, the upper bound depends on $\sqrt{\sum_k n_k^2 r_k^2}$ and the lower bound on $\max_k n_k r_k$. Since $\max_k n_k r_k \leq \sqrt{\sum_k n_k^2 r_k^2} \leq \sqrt{K}\cdot\max_k n_k r_k$, the gap between the bounds is at most $\sqrt{K}$, which is a small constant for typical land cover classifications ($K \leq 10$).

\begin{remark}[Practical interpretation]
\label{rem:piecewise_practical}
Given a scene with estimated per-stratum correlation ranges $\hat{r}_k$, the effective sample size should be computed as $n_{\mathrm{eff}} = n^4 / \sum_k n_k^2 \hat{r}_k^2$, not from a global variogram. A scene that is half cropland ($r = 40$) and half forest ($r = 5$) has $n_{\mathrm{eff}} = 2n^2/1625 \approx n^2/812$, whereas a global variogram might estimate $r \approx 22$, giving $n^2/484$, overestimating of the effective sample size by nearly a factor of two.
\end{remark}

\section{Implications for Spatial Cross-Validation}
\label{sec:cv}

The main practical consequence of the preceding results concerns validation methodology. We state implications for each setting.

\begin{corollary}[Spatial cross-validation in the stationary case]
\label{cor:cv_stationary}
Under the conditions of Theorem~\ref{thm:upper}:
\begin{enumerate}
\item \textit{Spatial block holdout} with block separation $g \geq r \cdot O(\log n)$ achieves the optimal generalization bound with effective sample size $n_{\mathrm{eff}} = \Theta(n^2/r^2)$ up to log factors.
\item \textit{Random holdout} treats all $n^2$ pixels as independent. Any uniform deviation bound based on the nominal sample size $n^2$ underestimates the true generalization error by a factor of up to $\sim r \log n$.
\end{enumerate}
\end{corollary}

\begin{proof}
Statement (i) is immediate from Theorem~\ref{thm:upper}. $\hat{R}_{\mathrm{block}}$ with separation $g = \Theta(r\log n)$ achieves the bound. For (ii), a random subset of $n^2/2$ test pixels drawn without spatial structure includes $\sim n^2/(2r^2)$ effectively independent observations (by the lower bound, this cannot yield better than $\Omega(r\sqrt{d_{\mathrm{vc}}}/n)$ accuracy in estimating $R(h)$), but the practitioner computes a standard error as if all $n^2/2$ pixels were independent, underestimating the confidence interval width by $\sqrt{n^2 / n_{\mathrm{eff}}} = \Theta(r)$.
\end{proof}

\begin{corollary}[Spatial cross-validation in the piecewise stationary case]
\label{cor:cv_piecewise}
Under the conditions of Theorem~\ref{thm:piecewise_upper}, spatial block holdout with per-stratum block separation $g_k = \Theta(r_k \log n)$ within each region~$\Lambda_k$ achieves the generalization bound~\eqref{eq:piecewise_bound} with effective sample size $n_{\mathrm{eff}} = n^4 / \sum_k n_k^2 r_k^2$. In particular:
\begin{enumerate}
\item Let $\bar{r} = \sum_k w_k r_k$ denote the area-weighted mean range, with $w_k = n_k^2/n^2$. Then
\[
n_{\mathrm{eff}} \;\le\; \frac{n^2}{\bar{r}^{\,2}},
\]
with equality if and only if all $r_k$ are identical. Thus, whenever the per-stratum ranges are heterogeneous, substituting any single averaged range into the stationary formula $n^2/\bar{r}^2$ strictly overestimates the effective sample size.
\item Random holdout underestimates the width of confidence intervals for the generalization error by a factor of up to $\sqrt{n^2/n_{\mathrm{eff}}} = \sqrt{\textstyle\sum_k n_k^2 r_k^2}\,/\,n$.
\end{enumerate}
\end{corollary}

\begin{proof}
Statement (i): since $\sum_k w_k = 1$ and $x \mapsto x^2$ is strictly convex, Jensen's inequality gives $\bar{r}^{\,2} = \bigl(\sum_k w_k r_k\bigr)^2 \le \sum_k w_k r_k^2 = \sum_k n_k^2 r_k^2 / n^2$, with equality if and only if the $r_k$ are all equal. Hence $n^2/\bar{r}^{\,2} \ge n^4/\sum_k n_k^2 r_k^2 = n_{\mathrm{eff}}$. Statement (ii) follows from the same argument as Corollary~\ref{cor:cv_stationary}(ii).
\end{proof}

\begin{remark}[Anisotropic block separation]
\label{rem:aniso_cv}
For anisotropic fields (Theorem~\ref{thm:aniso_upper}), the block separation should be direction-dependent: $g_k = \Theta(r_k \log n)$ in direction~$k$. This yields elongated validation blocks aligned with the correlation structure. For SAR imagery with range correlation $r_1 = 50$ pixels and azimuth correlation $r_2 = 10$ pixels, using isotropic separation $g = 50$ in both directions discards $\sim 80\%$ of the available effective samples compared to the anisotropic prescription.
\end{remark}

\paragraph{Practical workflow}
The theoretical results suggest the following procedure for spatial cross-validation on a new scene:
\begin{enumerate}
\item \textit{Stratify.} Partition the image into approximately homogeneous regions $\Lambda_1, \dots, \Lambda_K$ (e.g. by land cover class or ecological zone). If directional correlation is suspected, note the principal axes.
\item \textit{Estimate correlation ranges.} Within each region $\Lambda_k$, compute the empirical variogram $\hat{\gamma}_k(h)$ and fit a parametric model to obtain $\hat{r}_k$ (or directional ranges $\hat{r}_{k,1}, \hat{r}_{k,2}$ if anisotropic). A robust choice is the practical range, defined as the lag at which the fitted variogram reaches 95\% of its sill.
\item \textit{Compute effective sample size.} Use $n_{\mathrm{eff}} = n^4 / \sum_k n_k^2 \hat{r}_k^2$ (or $n^4 / \sum_k n_k^2 \hat{r}_{k,1}\hat{r}_{k,2}$ for anisotropic strata.
\item \textit{Set block separation.} Within each region, separate spatial CV folds by at least $g_k = c \cdot \hat{r}_k \log n$ pixels, where $c \geq 1/c_1$ depends on the covariance model.
\item \textit{Report corrected confidence intervals.} Replace the nominal sample size $n^2$ with $n_{\mathrm{eff}}$ when computing standard errors and confidence intervals for accuracy metrics.
\end{enumerate}
The bound is continuous in~$r$, and replacing $r_k$ with $\hat{r}_k(1 \pm \epsilon)$ changes $n_{\mathrm{eff}}$ by a factor of at most $(1 \pm \epsilon)^{-2}$, so moderate variogram estimation error produces proportional, not catastrophic, error in the bound.

\paragraph{Worked examples}
For a Sentinel-2 scene ($n = 10{,}000$ pixels per side at 10m resolution) with landscape correlation ranges of 200--500m ($r = 20$--$50$ pixels), the effective sample size under global stationarity is $n_{\mathrm{eff}} \approx n^2/r^2 \approx 40{,}000$--$250{,}000$, roughly two to three orders of magnitude smaller than the pixel count of $10^8$.
For a heterogeneous scene with three strata, suppose we have cropland ($40\%$ of pixels, $r_1 = 40$), forest ($40\%$, $r_2 = 8$), and urban ($20\%$, $r_3 = 15$). Writing the stratum pixel counts as $n_k^2 = w_k n^2$, the stratified effective sample size from~\eqref{eq:neff_piecewise} is
\begin{align*}
n_{\mathrm{eff}}
&= \frac{n^4}{0.4\, n^2 \cdot 40^2 + 0.4\, n^2 \cdot 8^2 + 0.2\, n^2 \cdot 15^2} \\
&= \frac{n^2}{0.4 \cdot 1600 + 0.4 \cdot 64 + 0.2 \cdot 225}
\approx \frac{n^2}{711}
\end{align*}
A global variogram on this scene might estimate the area-weighted mean range $\bar{r} \approx 22$, giving $n^2/\bar{r}^{\,2} \approx n^2/493$, which overestimates $n_{\mathrm{eff}}$ by a factor of $\sim 1.4$. Whenever the per-stratum ranges are heterogeneous, any single averaged range yields confidence intervals that are too narrow. Per-stratum estimation avoids this bias.

\section{Synthetic Experiments}
\label{sec:synthetic}

We first validate the theoretical results with a controlled simulation. The goals are to (1) confirm that random holdout overestimates classification accuracy on spatially correlated imagery, with the bias growing in the correlation range $r$; (2) verify the effective-sample-size scaling $n_{\mathrm{eff}} \sim n^{2}/r^{2}$; and (3) demonstrate that validation optimism decays with distance from the training region, reaching zero once the distance exceeds $r$.

We generate synthetic $n\times n$ images ($n{=}500$) from a stationary Gaussian random field with exponential covariance $C(d)=\exp(-d/r)$ and $d{=}6$ spectral bands, each drawn independently from the same covariance structure via circulant embedding (FFT). The true label at each pixel is $Y_{s}=\mathbf{1}[\mathbf{w}^{\!\top}\mathbf{X}_{s}+\epsilon_{s}>0]$, where $\mathbf{w}\in\mathbb{R}^{d}$ is a fixed unit-norm weight vector and $\epsilon_{s}\!\sim\!\mathcal{N}(0,\sigma^{2})$ with $\sigma{=}0.5$ is i.i.d.\ label noise. The linear boundary ensures that the Bayes-optimal classifier is well-defined, while the label noise makes the classification task nontrivial (Bayes error $\approx$15\%). The correlation range $r$ is the parameter we vary across experiments, taking values in $\{1,3,5,10,\ldots,70\}$.

We use a decision tree (depth 15, minimum leaf size 3) as the classifier. The high capacity of the tree is crucial here as a linear classifier trained on data generated from a linear boundary converges to the Bayes-optimal decision surface regardless of spatial structure and therefore exhibits negligible optimism. The decision tree can actually memorize spatially correlated noise patterns in the training set, and this overfitting to local spatial structure is the mechanism that inflates random-holdout accuracy.

True generalization error is estimated by evaluating each trained classifier on five independently generated images drawn from the same distribution (same $\mathbf{w}$, fresh random-field realizations). All reported results are averaged over 500 replications per configuration with error bars denoting $\pm 1$ standard error. Figure~\ref{fig:simulation} visualizes the experiments.

\begin{figure*}[htbp]
 \centering
 \includegraphics[width=\textwidth]{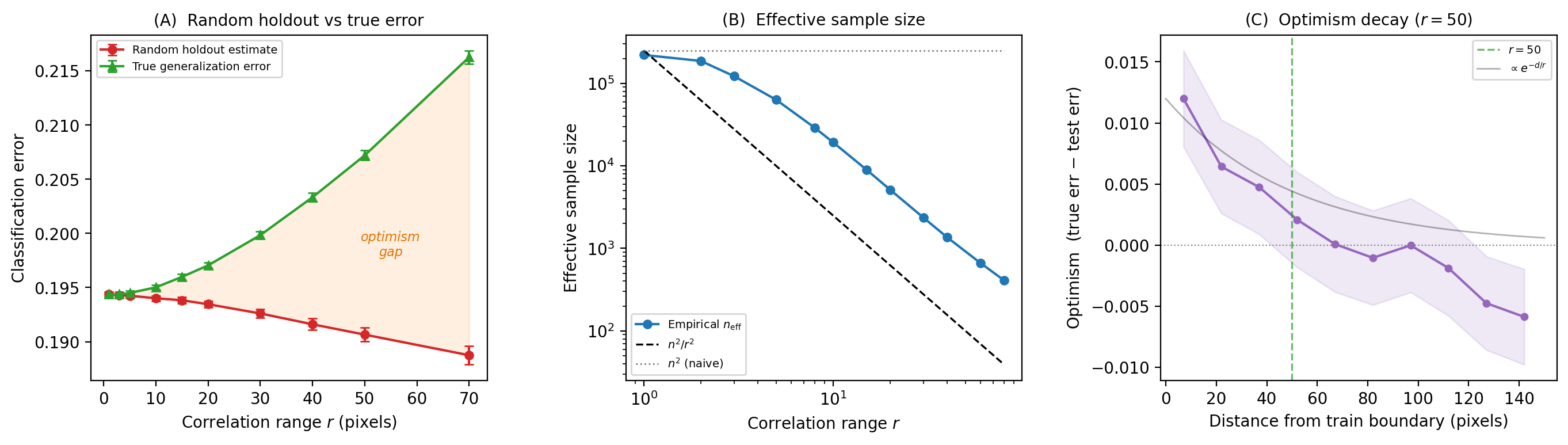}
 \caption{Simulation validation of the results.
 (A)~Random holdout error (red) diverges from generalization error (green) as the correlation range~$r$ increases. Shaded area is the optimism gap.
 (B)~Empirical effective sample size (blue) scales log-log slope $=-1.78$ vs.\ theory $-2$.
 (C)~Validation optimism decays with distance from the training boundary, reaching zero at $d\approx r$ (vertical line). The gray curve shows $\propto e^{-d/r}$.
 All experiments use $n{=}500$, $d{=}6$ spectral bands, exponential covariance.}
 \label{fig:simulation}
\end{figure*}

\paragraph{Panel~(A): Random holdout vs.\ true error} For each value of $r$, we train the decision tree on a random 50\% of the pixels in a single image and evaluate on the remaining 50\%. The red curve shows the random-holdout error estimate and the green curve shows the true generalization error measured on fresh images. At $r{=}1$ (nearly i.i.d.\ pixels) the two coincide. As $r$ increases, they diverge: the random-holdout estimate decreases while the true error increases. The shaded region between them is the optimism gap, namely the amount by which random holdout underestimates the true classification error. At $r{=}70$ the gap reaches approximately 2.8 percentage points. As anticipated, the holdout estimate falls (the interleaved test set rewards the memorization) while true error rises (more correlation, more structured noise to memorize).
 
\paragraph{Panel~(B): Effective sample size} To measure $n_{\mathrm{eff}}$ empirically we train a logistic regression classifier on one image and evaluate it on a second, independently generated image, repeating this 800 times per value of $r$. The variance of the test error across repetitions gives $n_{\mathrm{eff}} = p(1{-}p)/\mathrm{Var}(\hat{e})$, where $p$ is the mean error. We use logistic regression rather than a decision tree here to minimize the contribution of training-set variability to the total variance, isolating the evaluation-side effect. The empirical $n_{\mathrm{eff}}$ (blue) tracks the theoretical $n^{2}/r^{2}$ line (dashed) in slope, with a log-log slope of $-1.78$, the shortfall from the predicted $-2$ reflecting finite-$r$ curvature from the $log^2n$ term, which flattens the local slope when $r$ is not small relative to $n$. The constant-factor offset (roughly one order of magnitude) is consistent with the polylogarithmic factors.

\paragraph{Panel~(C): Optimism decay with distance} We train the decision tree on the left half of the image (columns $1$ to $n/2$) and measure the test error in narrow 15-pixel-wide bands at increasing distances from the training boundary. At each distance $d$, we compute the optimism as the difference between the true generalization error and the band-specific test error. Positive optimism indicates that the band benefits from spatial leakage. The optimism is largest immediately adjacent to the training region ($+1.2$\% at distance 7) and decays to zero by distance $\approx r{=}50$, closely following the exponential decay $\propto e^{-d/r}$ predicted by the covariance model (gray curve). The vertical dashed line marks $d{=}r$. Beyond $d \approx r$, the optimism fluctuates around zero, confirming that pixels separated from the training data by more than the correlation range provide an unbiased estimate of generalization performance. The slight negative drift at distances $d > 2r$ reflects a secondary effect. At large separations the test pixels inhabit a sufficiently different realization of the random field that the classifier, which has partially memorized the training-half realization, performs slightly worse than on a generic fresh image. This is the distribution shift phenomenon we discuss in connection with Wadoux~et~al.~\cite{wadoux2021spatial}.

\section{Empirical Experiments}
\label{sec:real}

We then validate the theoretical predictions on satellite imagery spanning seven ecoregions, three sensors, and diverse land cover types. All tiles are approximately $500 \times 500$ pixels, extracted via Google Earth Engine \citep{gorelick2017google} from cloud-free summer (May--September 2021) median composites, with land cover labels from the 2021 National Land Cover Database (NLCD)~\citep{USGS2024NLCD}.

We use three sensors: Landsat~8 surface reflectance~\citep{roy2014landsat8} (30m, six spectral bands), Sentinel-2 surface reflectance~\citep{drusch2012sentinel2} (10m, ten bands), and Sentinel-1 GRD SAR~\citep{torres2012sentinel1} (10m, VV/VH polarizations). Tiles are drawn from seven regions selected for ecological diversity. We use Iowa (flat cropland), Colorado Front Range (mountainous forest/shrub), Chesapeake Bay (coastal wetland/forest/cropland), Arizona (arid desert/shrub), New England (mixed deciduous/urban), Georgia (pine plantation/wetland), and California Central Valley (irrigated agriculture). In total, 43 tiles were analyzed with 25 Landsat, 10 Sentinel-2 (co-located with Landsat tiles for cross-sensor comparison), and 8 Sentinel-1 (one per region, for anisotropy analysis). Table~\ref{tab:summary} summarizes the results.

\begin{table}[t]
\centering
\caption{Summary of correlation ranges and effective sample sizes across all study tiles. $\hat{r}$: fitted exponential range from the NDVI (or VV) variogram. $n_{\mathrm{eff}}$: block-bootstrap effective sample size at block size $b = \hat{r}$.}
\label{tab:summary}
\fontsize{6.5}{10}\selectfont
\begin{tabular}{llcrrrrr}
\toprule
Region & Tile & Sensor & Scale & $\hat{r}$ (px) & $\hat{r}$ (m) & $n_{\mathrm{test}}$ & $n_{\mathrm{eff}}$ \\
\midrule
\multirow{11}{*}{Iowa}
 & IA-L1 & Landsat~8 & 30m & 15.0 & 449 & 145684 & 1930 \\
 & IA-L2 & Landsat~8 & 30m & 25.0 & 751 & 139876 & 610 \\
 & IA-L3 & Landsat~8 & 30m & 13.0 & 390 & 147744 & 2687 \\
 & IA-L4 & Landsat~8 & 30m & 23.2 & 696 & 141620 & 1861 \\
 & IA-L5 & Landsat~8 & 30m & 5.6 & 168 & 149865 & 11915 \\
 & IA-L6 & Landsat~8 & 30m & 20.8 & 625 & 142002 & 2463 \\
 & IA-L7 & Landsat~8 & 30m & 18.2 & 545 & 144530 & 2448 \\
 & IA-S2a & Sentinel-2 & 10m & 35.8 & 358 & 135520 & 955 \\
 & IA-S2b & Sentinel-2 & 10m & 19.4 & 194 & 142485 & 1215 \\
 & IA-S1a & Sentinel-1 & 10m & 18.0 & 180 & 144232 & 1384 \\
 & IA-S1b & Sentinel-1 & 10m & 11.9 & 119 & 146349 & 2467 \\
\midrule
\multirow{5}{*}{Colorado}
 & CO-L1 & Landsat~8 & 30m & 26.1 & 782 & 143674 & 1758 \\
 & CO-L2 & Landsat~8 & 30m & 14.3 & 430 & 149108 & 3168 \\
 & CO-L3 & Landsat~8 & 30m & 15.4 & 461 & 148632 & 2691 \\
 & CO-S2 & Sentinel-2 & 10m & 36.9 & 369 & 139095 & 1285 \\
 & CO-S1 & Sentinel-1 & 10m & 15.7 & 157 & 149238 & 2677 \\
\midrule
\multirow{5}{*}{Chesapeake}
 & CB-L1 & Landsat~8 & 30m & 48.2 & 1445 & 138308 & 1608 \\
 & CB-L2 & Landsat~8 & 30m & 11.5 & 345 & 156006 & 6235 \\
 & CB-L3 & Landsat~8 & 30m & 101.1 & 3034 & 112497 & 171 \\
 & CB-S2 & Sentinel-2 & 10m & 17.8 & 178 & 153405 & 2810 \\
 & CB-S1 & Sentinel-1 & 10m & 14.9 & 149 & 154866 & 4167 \\
\midrule
\multirow{5}{*}{Arizona}
 & AZ-L1 & Landsat~8 & 30m & 10.7 & 321 & 157464 & 2355 \\
 & AZ-L2 & Landsat~8 & 30m & 18.9 & 566 & 153260 & 1740 \\
 & AZ-L3 & Landsat~8 & 30m & 43.3 & 1299 & 141135 & 4269 \\
 & AZ-S2 & Sentinel-2 & 10m & 28.6 & 286 & 148716 & 373 \\
 & AZ-S1 & Sentinel-1 & 10m & 11.4 & 114 & 156978 & 1898 \\
\midrule
\multirow{5}{*}{New England}
 & NE-L1 & Landsat~8 & 30m & 28.6 & 859 & 144248 & 659 \\
 & NE-L2 & Landsat~8 & 30m & 12.0 & 360 & 154256 & 4330 \\
 & NE-L3 & Landsat~8 & 30m & 21.2 & 635 & 148500 & 2027 \\
 & NE-S2 & Sentinel-2 & 10m & 32.1 & 321 & 142560 & 529 \\
 & NE-S1 & Sentinel-1 & 10m & 23.0 & 230 & 147015 & 472 \\
\midrule
\multirow{6}{*}{Georgia}
 & GA-L1 & Landsat~8 & 30m & 35.2 & 1057 & 139173 & 842 \\
 & GA-L2 & Landsat~8 & 30m & 8.4 & 253 & 156806 & 5065 \\
 & GA-L3 & Landsat~8 & 30m & 21.5 & 646 & 148004 & 6550 \\
 & GA-S2a & Sentinel-2 & 10m & 101.7 & 1017 & 112440 & 91 \\
 & GA-S2b & Sentinel-2 & 10m & 18.9 & 189 & 152312 & 958 \\
 & GA-S1 & Sentinel-1 & 10m & 70.9 & 709 & 127342 & 134 \\
\midrule
\multirow{6}{*}{California}
 & CA-L1 & Landsat~8 & 30m & 37.6 & 1127 & 146421 & 955 \\
 & CA-L2 & Landsat~8 & 30m & 23.3 & 698 & 153634 & 1842 \\
 & CA-L3 & Landsat~8 & 30m & 14.6 & 439 & 157760 & 2621 \\
 & CA-S2a & Sentinel-2 & 10m & 25.3 & 253 & 152337 & 858 \\
 & CA-S2b & Sentinel-2 & 10m & 23.7 & 237 & 153634 & 694 \\
 & CA-S1 & Sentinel-1 & 10m & 25.6 & 256 & 152337 & 638 \\
\bottomrule
\end{tabular}
\end{table}

\subsection{Variogram Estimation and $n^2/r^2$ Scaling}

For each tile we compute the NDVI (or VV backscatter for SAR) and estimate an isotropic empirical variogram from ${\sim}200{,}000$ random pixel pairs, binned at 2-pixel lag intervals. We fit an exponential model with nugget, $\gamma(h) = c_0 + (c_1 - c_0)(1 - e^{-h/r})$, by nonlinear least squares.

The fitted correlation ranges vary substantially both within and across regions. For Landsat tiles, the median $\hat{r}$ per region ranges from ${\sim}15$ pixels (450m) in Iowa to ${\sim}43$ pixels (1300m) in parts of Arizona, with within-region variability often spanning a factor of 2--3. Iowa Landsat tiles yield $\hat{r} \in [5.6, 25.0]$ pixels across seven tiles, reflecting local differences in field size and land cover heterogeneity. Cross-sensor comparisons on co-located tiles show that physical correlation lengths are broadly consistent between Landsat and Sentinel-2 (e.g. Iowa: ${\sim}450$m vs.\ ${\sim}280$m), while SAR consistently shows shorter correlation lengths (${\sim}150$--$180$m), consistent with radar sensitivity to fine-scale surface structure.

To test the $n^2/r^2$ prediction, we employ a spatial block bootstrap on all tiles. For each tile, a logistic regression classifier is trained on the left third of the image (with a gap of $\hat{r}$ pixels to prevent leakage) and evaluated on the remaining test region. For each block size $b$, we tile the test region into non-overlapping $b \times b$ blocks, resample blocks with replacement 500 times, and compute $n_{\mathrm{eff}} = \hat{p}(1-\hat{p})/\mathrm{Var}(\hat{e})$.

Fig.~\ref{fig:real} shows the result for four representative tiles spanning Iowa, Chesapeake, California, and New England. In all cases, $n_{\mathrm{eff}}$ decreases as a power law in the block size, running parallel to the $b^{-2}$ reference line. The log-log slopes range from $-1.6$ to $-1.9$ across tiles, consistent with the theoretical prediction of $-2$ up to the polylogarithmic correction in Theorem~\ref{thm:upper}. The scaling holds across all three sensors, all seven regions, and for both binary and multi-class classification tasks. Across all tiles, $n_{\mathrm{eff}}$ ranges from 91 (GA-S2a, $\hat{r} = 102$ pixels) to 11915 (IA-L5, $\hat{r} = 5.6$ pixels), compared to test set sizes of 112000--158000 pixels, confirming that effective sample sizes are typically 1--3 orders of magnitude smaller than pixel counts.

\begin{figure}[t]
\centering
\includegraphics[width=\columnwidth]{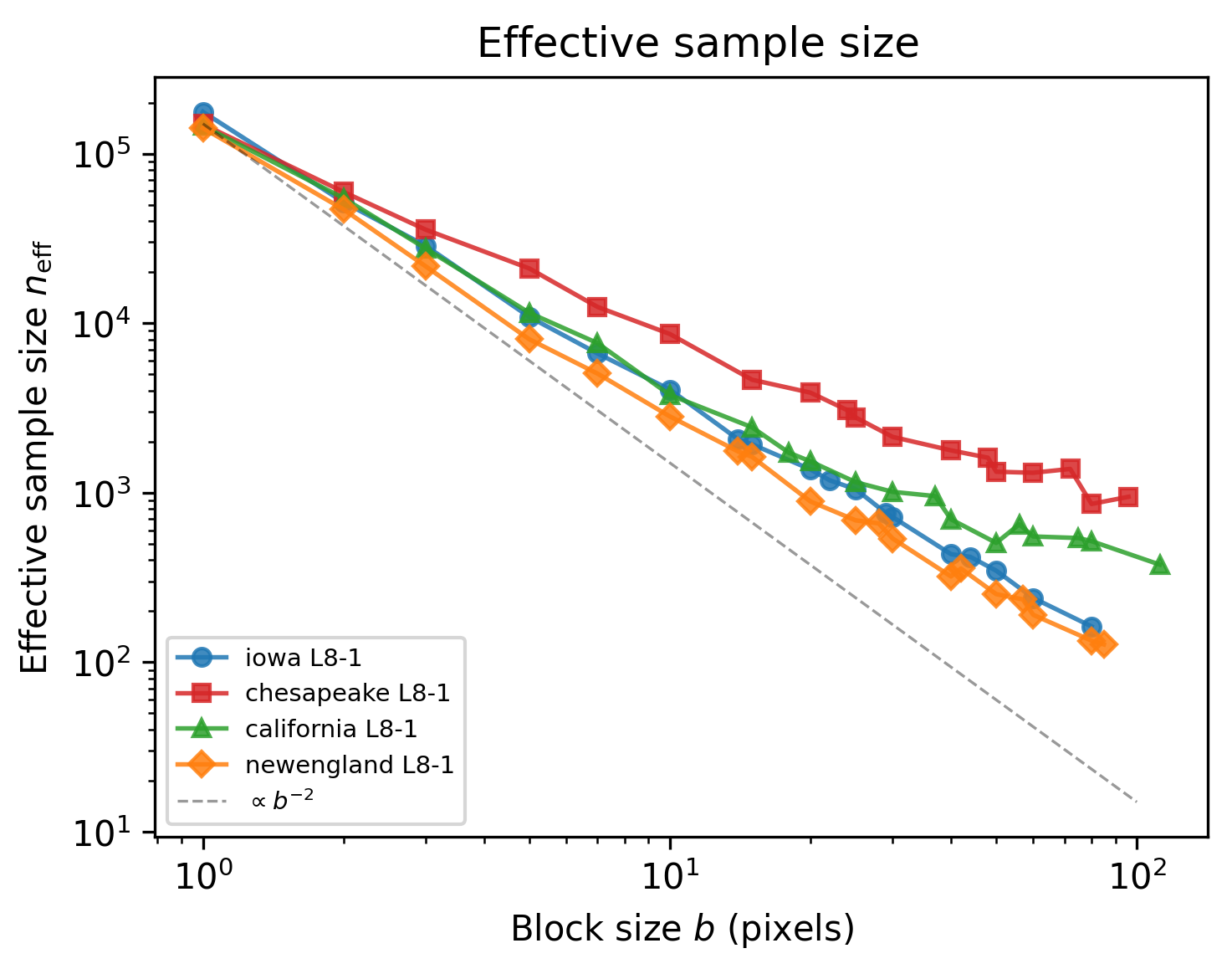}
\caption{Block-bootstrap effective sample size as a function of block size~$b$ for four tiles spanning different regions.}
\label{fig:real}
\end{figure}

\subsection{Anisotropy Validation}

Sentinel-1 SAR imagery provides a natural test case for the anisotropic extension (Theorem~\ref{thm:aniso_upper}), since SAR backscatter can exhibit directional correlation structure from both the sensor geometry and landscape features. For each of the eight SAR tiles, we compute directional variograms along the N--S and E--W axes (with a $\pm 30^\circ$ angular tolerance) and fit exponential models independently to obtain directional ranges $\hat{r}_1$ (N--S) and $\hat{r}_2$ (E--W).

Table~\ref{tab:aniso} summarizes the results. The anisotropy ratio $\max(\hat{r}_1, \hat{r}_2)/\min(\hat{r}_1, \hat{r}_2)$ ranges from 1.0 (isotropic) to 8.2, and the fraction of effective samples wasted by isotropic blocking ranges from 0\% to 88\%. Chesapeake exhibits the strongest anisotropy (ratio 8.2x, 88\% waste), with the E--W correlation extending along the coastal shoreline ($\hat{r}_2 = 73$ pixels) while the N--S direction decorrelates rapidly across the land--water boundary ($\hat{r}_1 = 9$ pixels). Colorado shows a similar pattern along mountain ridges (ratio 5.9x, 83\% waste). California and Iowa are near-isotropic (ratios 1.1x and 1.4x), consistent with the regular geometry of agricultural fields. These results confirm that the anisotropic prescription $n_{\mathrm{eff}} = n^2/(r_1 r_2)$ from Theorem~\ref{thm:aniso_upper} can recover substantially more effective samples than isotropic blocking with $r_{\max}$ in landscapes with directional structure.

\begin{table}[htbp]
\centering
\caption{Directional correlation ranges from Sentinel-1 SAR. $\hat{r}_1$, $\hat{r}_2$: N--S and E--W exponential ranges. Ratio: $\max/\min$ of the two ranges. Waste: fraction of effective samples lost by using isotropic blocking with $r_{\max}$ instead of the anisotropic $r_1 r_2$.}
\label{tab:aniso}
\small
\begin{tabular}{lrrrrr}
\toprule
Tile & $\hat{r}_1$ (px) & $\hat{r}_2$ (px) & Ratio & $r_1 r_2$ & Waste \\
\midrule
IA-S1a & 17 & 12 & 1.4x & 208 & 27\% \\
IA-S1b & 75 & 75 & 1.0x & 5563 & 0\% \\
CO-S1 & 60 & 10 & 5.9x & 604 & 83\% \\
CB-S1 & 9 & 73 & 8.2x & 643 & 88\% \\
AZ-S1 & 9 & 21 & 2.3x & 185 & 57\% \\
NE-S1 & 47 & 36 & 1.3x & 1696 & 22\% \\
GA-S1 & 159 & 60 & 2.7x & 9515 & 63\% \\
CA-S1 & 24 & 27 & 1.1x & 647 & 10\% \\
\bottomrule
\end{tabular}
\end{table}

\subsection{Piecewise Stationarity Validation}
 
To validate the piecewise stationary extension (Theorem~\ref{thm:piecewise_upper}), we partition selected tiles by NLCD land cover class and fit per-stratum variograms. For strata with at least 500 valid pixels, we estimate $\hat{r}_k$ and compute the stratified effective sample size $n_{\mathrm{eff}} = n_{\mathrm{total}}^2 / \sum_k n_k \hat{r}_k^2$, where $n_k$ is the pixel count in stratum~$k$. Strata with non-convergent variogram fits are excluded.
 
Table~\ref{tab:piecewise} shows the per-stratum breakdown for six tiles from Iowa, Colorado, and Chesapeake. The correlation ranges vary dramatically within individual scenes. In Iowa, cropland has $\hat{r} = 9.7$ pixels while developed land has $\hat{r} = 3.4$, and in Colorado, herbaceous cover reaches $\hat{r} = 60$ pixels while shrub is only $4.6$. The dominant stratum's correlation range does not determine the global estimate.

The ratio of stratified to global $n_{\mathrm{eff}}$ ranges from $0.4\times$ (CO-L3) to $5.2\times$ (CO-L2), demonstrating that the global variogram can err in either direction. When most strata have shorter ranges than the global estimate (IA-L1, CO-L1), the stratified $n_{\mathrm{eff}}$ is higher, so the global variogram is overly conservative. When a minority stratum with long-range correlation is averaged away (CO-L3), the global estimate is anticonservative and the true effective sample size is lower than reported. The Chesapeake tile (CB-L1) illustrates a third scenario where water ($\hat{r} = 32$, 52\% of pixels) and wetland ($\hat{r} = 32$, 12\%) dominate the denominator despite cropland ($\hat{r} = 7$) covering nearly a quarter of the scene.
 
These results confirm that per-stratum estimation is necessary for heterogeneous scenes and that the formula from Theorem~\ref{thm:piecewise_upper} captures effects invisible to a global variogram.

\begin{table}[htbp]
\centering
\caption{Stratified vs.\ global effective sample size for six heterogeneous tiles. $\bar{r}$: global variogram range. $r_k$ range: min--max of per-stratum ranges (NLCD classes with $\geq 500$ pixels).}
\label{tab:piecewise}
\small
\begin{tabular}{lcrrrr}
\toprule
Tile & $K$ & $\bar{r}$ (px) & $r_k$ range & $n_{\mathrm{eff}}^{\mathrm{global}}$ & $n_{\mathrm{eff}}^{\mathrm{strat}}$ / $n_{\mathrm{eff}}^{\mathrm{global}}$ \\
\midrule
IA-L1 & 6 & 14.0 & 3.2--17.6 & 1161 & 2.4x \\
CO-L1 & 5 & 17.9 & 4.6--32.1 & 730 & 4.3x \\
CO-L2 & 5 & 28.2 & 3.4--28.8 & 295 & 5.2x \\
CO-L3 & 6 & 17.1 & 11.3--60.3 & 802 & 0.4x \\
CO-S2 & 4 & 21.2 & 2.2--17.7 & 522 & 2.2x \\
CB-L1 & 5 & 50.3 & 6.6--32.4 & 96 & 3.5x \\
\bottomrule
\end{tabular}
\end{table}

\section{Discussion}
\label{sec:discussion}
 
We established that the effective sample size of a spatially correlated remote sensing image is $\Theta(n^2/r^2)$, where $r$ is the correlation range, and that this rate is minimax optimal. The result provides a rigorous statistical justification for spatial cross-validation in the remote sensing context and gives quantitative guidance on the required block separation.
 
The bound takes the correlation range~$r$ as known. In practice, $r$ can be estimated from an empirical variogram. We can compute the semivariance $\hat{\gamma}(h) = \frac{1}{2|N(h)|}\sum_{(s,t) \in N(h)} (X_s - X_t)^2$ at each lag~$h$, fit a parametric model (exponential, Mat\'ern, etc.), and read off the practical range at which correlation drops to 5\% of the sill. Since the bound is continuous in~$r$, moderate estimation error in~$\hat{r}$ produces proportional error in the bound by replacing $r$ with $\hat{r}(1 \pm \epsilon)$ changes $n_{\mathrm{eff}}$ by a factor of $(1 \pm \epsilon)^{-2}$.

A subtlety concerns which spatial field the correlation range~$r$ refers to. The bounds are driven by the mixing rate of the loss field $\{\ell(h,s)\}$, whereas the variograms in Section~\ref{sec:real} are computed on the raw features, which may not coincide. The feature range bounds the loss range from above when the label depends on the features through a fixed decision boundary. However, when intra-class variability is small, the loss field decorrelates faster than the feature field, and the feature-variogram range $\hat{r}$ overestimates the effective correlation range of the loss, yielding a conservative $n_{\mathrm{eff}}$. Estimating the loss-field range directly, e.g.\ from the variogram of held-out residuals, would tighten the bound, at the cost of requiring a fitted classifier. We use the feature range throughout as it is model-free, computable prior to training, and its conservatism errs in the safe direction for confidence intervals.

Wadoux et al.~\cite{wadoux2021spatial} argued that spatial cross-validation can be pessimistic for map accuracy assessment. This is consistent with our analysis once the prediction task is distinguished. Theorem~\ref{thm:upper} characterizes generalization to new draws from the same stationary field. When the goal is instead extrapolation to geographically disjoint regions (new biomes, different continents), the relevant difficulty is distributional shift, not spatial correlation, and domain adaptation bounds are more appropriate. Spatial CV with large block separations can conflate the two effects, penalizing a classifier for distribution shift that the user does not intend to generalize across. Our result clarifies that the block separation should be scaled to the correlation range~$r$, not to the size of the study area.
 
We conjecture the blocking argument generalizes to space-time lattices $\Lambda_n \times \{1,\dots,T\}$ with anisotropic mixing spatial range~$r_s$ and temporal range~$r_t$. Applying 3D blocking gives an effective sample size $n_{\mathrm{eff}} \sim n^2 T / (r_s^2 r_t)$, quantifying the marginal value of each additional revisit. For a Sentinel-2 time series with $T = 36$ revisits per year and temporal correlation of $r_t \approx 3$ revisits, a temporal stack multiplies the effective sample size by $\sim T/r_t \approx 12$ relative to a single image.
 
Our results have several limitations. The bounds are worst-case and conservative, with unoptimized constants and polylogarithmic factors. The key value is the $n_{\mathrm{eff}} = \Theta(n^2/r^2)$ scaling and the spatial-CV guidance, not a sharp constant. The correlation range $r$ must be estimated from an empirical variogram; while the bound is continuous in $r$ and moderate estimation error propagates proportionally (Remark~\ref{rem:piecewise_practical}), a poorly specified variogram will bias $n_{\mathrm{eff}}$. We assume label noise is conditionally independent given features, but tasks with spatial label structure (e.g. segmentation) may exhibit different mixing behavior. Investigating such settings is left to future work.
\section{Conclusion}

We derived finite-sample generalization bounds for classifiers trained on spatially dependent remote sensing data. The main result is that the effective sample size of an $n \times n$ image with spatial correlation range $r$ is $\Theta(n^2/r^2)$, and that this rate is minimax optimal. For anisotropic fields, the effective sample size is $\Theta(n^2/(r_1 r_2))$, and for piecewise stationary scenes per-stratum correlation ranges must be combined through a variance-weighted formula. As a corollary, spatial block holdout with separation proportional to $r$ achieves optimal generalization guarantees, while random holdout can underestimate confidence intervals by $\Theta(r)$. These results provide theoretical footing and quantitative guidance for spatial cross-validation in remote sensing accuracy assessments.

\bibliographystyle{IEEEtran}
\bibliography{references}
\clearpage

\appendix
\setcounter{section}{0}
\renewcommand{\thesection}{S\arabic{section}}
\renewcommand{\theequation}{S\arabic{equation}}
\renewcommand{\thefigure}{S\arabic{figure}}
\renewcommand{\thetable}{S\arabic{table}}
\setcounter{equation}{0}\setcounter{figure}{0}\setcounter{table}{0}

\section{A Primer on Generalization Bounds}
\label{sec:si_primer}
This section briefly glosses the learning-theoretic machinery used in the main text, for readers who work with machine learning in practice but have not seen the underlying theory. Readers seeking a fuller treatment may consult the accessible introduction of Shalev-Shwartz and Ben-David~\cite{shalevshwartz2014understanding}, whose early chapters develop PAC learning, VC dimension, and uniform convergence from first principles (a free copy is hosted by the authors \href{https://www.cs.huji.ac.il/~shais/UnderstandingMachineLearning/}{here}), or the more advanced account of Mohri et al.~\cite{mohri2018foundations}, which also covers the learning theoretic treatment of dependent data underlying the mixing bounds we build on.

Anyone who has held out a test set has estimated a generalization gap. A classifier $h$ has a \emph{true risk} $R(h)$, or its error rate on the whole population, which we never observe, and an \emph{empirical risk} $\hat{R}(h)$, which is the error rate on the finite sample in hand. Reporting held-out accuracy is the act of using $\hat{R}(h)$ as a stand-in for $R(h)$, and the hope that the two are close is the hope that the gap $R(h) - \hat{R}(h)$ is small. A \emph{generalization bound} makes that hope quantitative as a guarantee, holding with high probability, on how large the gap can be and, crucially, one that holds \emph{simultaneously} for every $h$ the training procedure might have selected. If we only controlled the gap for a single $h$ fixed in advance, the training procedure's freedom to select the $h$ that happens to look best on this particular sample would be exactly the loophole overfitting exploits. The guarantee must therefore survive a worst-case choice of $h$, which is what the supremum in~\eqref{eq:vc_bound} encodes.

Everyone knows that a model flexible enough to memorize the training set will report an optimistically low training error, and VC dimension is the formal handle on ``flexible enough.'' Say that $\mathcal{H}$ \emph{shatters} a set of points if, for \emph{every} possible assignment of binary labels to them, some $h \in \mathcal{H}$ classifies all of them correctly, then the class is flexible enough to realize any labeling, including the adversarial ones. The \emph{VC dimension} $d_{\mathrm{vc}}$ of $\mathcal{H}$~\cite{vapnik1971vc} is the size of the largest set it can shatter. A linear classifier in the plane has $d_{\mathrm{vc}} = 3$, which means any labeling of three points is separable by a line, but no line separates four points labeled in the XOR pattern. A deep decision tree can shatter large sets and so has large $d_{\mathrm{vc}}$. A class that shatters sets of every size has $d_{\mathrm{vc}} = \infty$ and can fit pure noise. It is a measure of capacity that, unlike a raw parameter count, directly predicts overfitting. The classical guarantee is that, for $n$ i.i.d.\ samples, with probability at least $1-\delta$,
\begin{equation}
\label{eq:vc_bound}
\underbrace{\sup_{h \in \mathcal{H}} \bigl|R(h) - \hat{R}(h)\bigr|}_{\text{worst-case gap over the class}} \;\lesssim\; \sqrt{\frac{\overbrace{d_{\mathrm{vc}} \log n}^{\text{capacity}} + \log(1/\delta)}{\underbrace{n}_{\text{sample size}}}}.
\end{equation}

Holding $n$ fixed, a more flexible class (larger $d_{\mathrm{vc}}$) admits a larger gap, which is the formal version of ``flexible models overfit more.'' Holding $d_{\mathrm{vc}}$ fixed, the gap shrinks as $1/\sqrt{n}$, e.g. quadrupling the sample halves the error bar. The rate $1/\sqrt{n}$, and specifically the appearance of $n$ in the denominator, is the object our main results correct.

The $n$ in~\eqref{eq:vc_bound} arrives through a concentration inequality. The average of $n$ independent bounded quantities concentrates around its mean at rate $1/\sqrt{n}$. This is the one step that assumes independence, and it is the step spatial correlation breaks. When neighboring pixels are correlated, the $n^2$ pixels of an image do not supply $n^2$ independent quantities to average, adjacent pixels largely repeat information already present in their neighbors. The image behaves as though it held some smaller number $m < n^2$ of independent observations. The technical contribution of Section~\ref{sec:upper} is to make ``some smaller number'' formally defined. 

A blocking construction extracts $m \approx n^2/r^2$ nearly-independent units, and the bound~\eqref{eq:vc_bound} then holds with $m$ in place of $n$. $m$ here is not a subset of privileged pixels one could point to on the image, but the number of actually independent observations whose averaged error would carry the same variance as the full correlated scene. This variance-based reading is what the piecewise-stationary formula of Section~\ref{sec:extensions} makes explicit, where $n_{\mathrm{eff}}$ appears as the reciprocal of a weighted variance. 

The blocking argument shows the $n^2/r^2$ rate is achievable. Section~\ref{sec:lower} shows it cannot be beaten, and the reason is that an adversary may build a scene in which every $r \times r$ patch is internally identical and drawn independently of the others, then all $n^2$ pixels encode only $\lfloor n/r \rfloor^2$ free pieces of information, and no algorithm, however clever, can recover structure that was never present. The effective sample size is thus a property of the spatial dependence, not of the learner.

\section{Multi-Class Scaling Confirmation}
\label{sec:si_multiclass}

The generalization bounds of Section~\ref{sec:upper} are stated for a hypothesis class of $K$-way classifiers and hold for any number of classes $K \geq 2$. The number of classes enters the bound only through the VC dimension $d_{\mathrm{vc}}$, which sits inside the constant. The dependence of the effective sample size on the correlation range, $n_{\mathrm{eff}} \sim n^2/r^2$, is independent of $K$. This section confirms that invariance empirically, and additionally checks that class imbalance does not alter the scaling.

\subsection{Setup}
We reuse the field-generation and estimation procedure of the binary effective-sample-size experiment (Section~\ref{sec:synthetic}, Panel~B) unchanged, varying only the labeling rule. For each value of the correlation range $r$, we draw an $n \times n$ image ($n = 500$) with $d = 6$ spectral bands from a stationary Gaussian random field with exponential covariance $C(d) = \exp(-d/r)$, each band generated independently by circulant embedding. The latent score $\ell_s = \mathbf{w}^\top \mathbf{X}_s + \epsilon_s$ is formed from a fixed weight vector $\mathbf{w}$ and i.i.d.\ label noise $\epsilon_s \sim \mathcal{N}(0, 0.5^2)$, exactly as in the binary experiment.

The three labeling rules differ only in how $\ell_s$ is discretized:
\begin{itemize}
\item \emph{Binary} ($K=2$): a single threshold at the median of $\ell_s$, reproducing the rule $Y_s = \mathbf{1}[\ell_s > t]$ of Section~\ref{sec:synthetic}.
\item \emph{Four-class balanced} ($K=4$): three thresholds at the $25$th, $50$th, and $75$th percentiles of $\ell_s$, giving four classes of equal prior probability.
\item \emph{Four-class imbalanced} ($K=4$): three thresholds at the $70$th, $80$th, and $90$th percentiles, giving a $70/10/10/10$ split in which one class dominates. This is the regime in which practitioners most often worry about imbalance.
\end{itemize}
Thresholds are estimated once from a pilot draw so that the target class proportions are met on average. Each condition is fit with multinomial logistic regression and evaluated on a second, independently drawn image. We use $600$ replications per value of $r \in \{1, 2, 3, 5, 8, 10, 15, 20, 30, 40, 60, 80\}$ and estimate $n_{\mathrm{eff}} = \bar{p}(1-\bar{p})/\mathrm{Var}(\hat{e})$ from the across-replication variance of the test error, where $\bar{p}$ is the mean error rate. This estimator is valid for any $K$ since the per-pixel error indicator is Bernoulli regardless of how many classes produced a given misclassification.

\subsection{Results}
Figure~\ref{fig:si_multiclass} shows the three curves. In every condition, $n_{\mathrm{eff}}$ decreases as a power law in $r$, running parallel to the $n^2/r^2$ reference and far below the naive pixel count $n^2$. The fitted log-log slopes (over $r \geq 3$, where finite-$r$ curvature is mild) are $-1.85$ for the binary task, $-1.87$ for the four-class balanced task, and $-1.88$ for the four-class imbalanced task, and all consistent with the theoretical exponent of $-2$ up to the polylogarithmic correction of Theorem~\ref{thm:upper}.

The level of $n_{\mathrm{eff}}$ does depend on the task, as expected. The mean error rate rises from $\bar{p} \approx 0.13$ in the binary case to $\bar{p} \approx 0.33$ for four balanced classes (four classes are intrinsically harder to separate), and returns to $\bar{p} \approx 0.13$ for the imbalanced split, whose dominant class is easily classified. These shifts move the curves vertically but leave their slope unchanged. Neither the number of classes nor the degree of class imbalance affects the $n^2/r^2$ scaling. The cost of spatial correlation is governed by the correlation range $r$ alone, with $K$ and the class distribution absorbed into the constant, exactly as the theory predicts.

Class imbalance is a well-recognized difficulty in remote-sensing land-cover classification, where minority classes are systematically under-learned and are commonly addressed by resampling or margin-based methods~\cite{douzas2019imbalanced, mellor2015imbalance}. Our experiment shows that this concern, while important for per-class accuracy, is orthogonal to the spatial-correlation correction studied here. Tmbalance shifts the level of $n_{\mathrm{eff}}$ through the base error rate but leaves the $n^2/r^2$ scaling unchanged, so the block-separation prescription applies regardless of class balance.

\begin{figure}[t]
\centering
\includegraphics[width=\columnwidth]{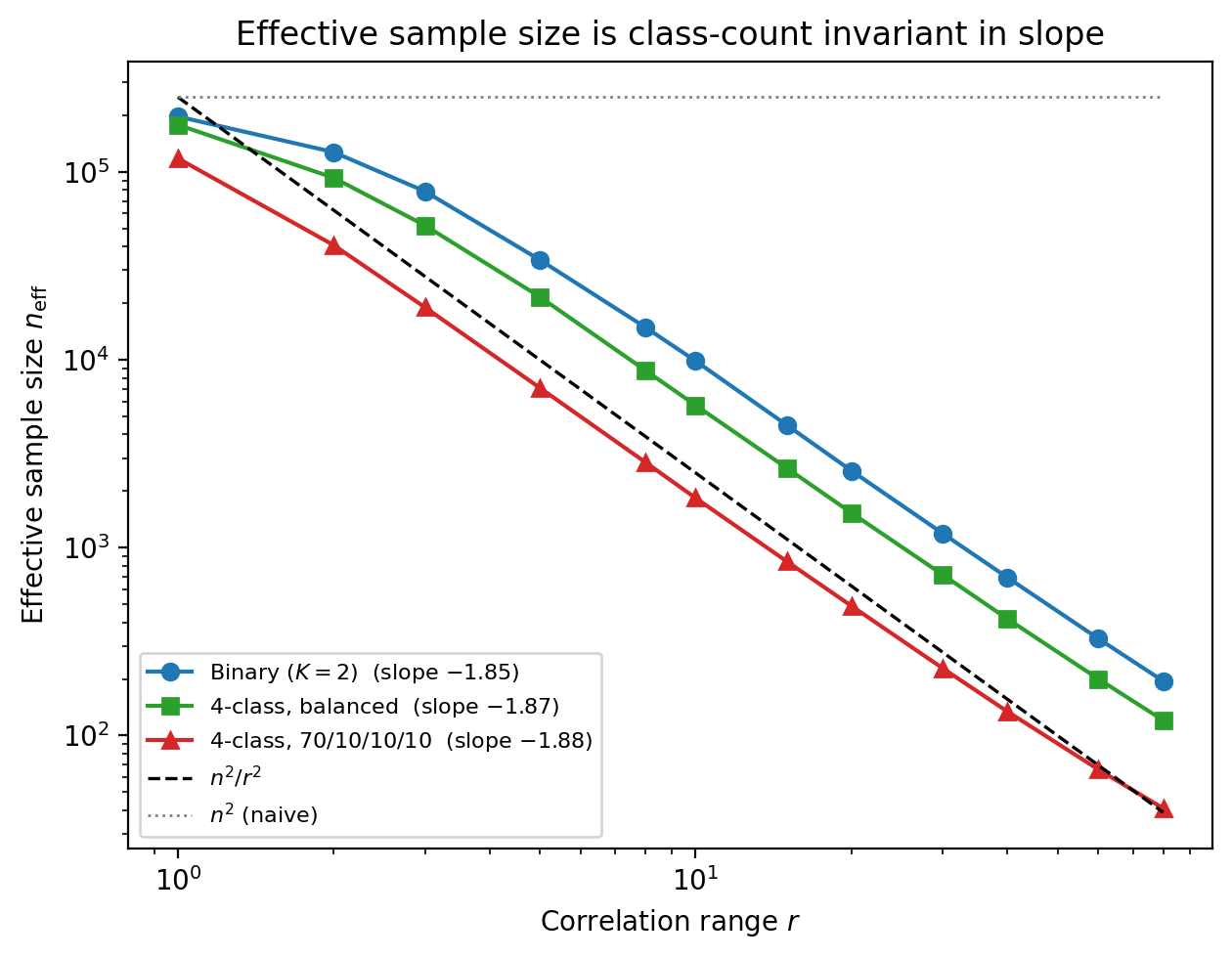}
\caption{Effective sample size versus correlation range for binary and four-class classification at $n=500$.}
\label{fig:si_multiclass}
\end{figure}

\end{document}